\documentclass[lettersize,journal]{IEEEtran}
\usepackage{amsmath,amsfonts}
\usepackage{algorithmic}
\usepackage{algorithm}
\usepackage{array}
\usepackage[caption=false,font=normalsize,labelfont=sf,textfont=sf]{subfig}
\usepackage{textcomp}
\usepackage{stfloats}
\usepackage{url}
\usepackage{verbatim}
\usepackage{graphicx}
\usepackage{cite}
\usepackage{amsfonts}  % blackboard math symbols, e.g., \mathbb
\usepackage{algorithm}
\usepackage{algorithmic}
\usepackage{siunitx}   % 用于书写 数字+单位 的格式
\usepackage{upgreek}   % 用于书写 数字+单位 的格式
\usepackage{xcolor}    % colors
\usepackage{booktabs}       % professional-quality tables
\usepackage{longtable}
\usepackage{threeparttable}
\usepackage{multirow}
\usepackage{multicol}
\usepackage{times}
\usepackage{hyperref}

\usepackage{amsthm}
\newtheorem{theorem}{Theorem}
\newtheorem{lemma}{Lemma}
\newtheorem{definition}{Definition}

\newcommand*{\dif}{\mathop{}\!\mathrm{d}}
\newcommand*{\tr}{\mathop{}\!\mathrm{tr}}

\newcommand{\ff}[1] {\textcolor{black}{#1}}

\begin{document}
	
\title{\huge\bf Neural Network Gaussian Processes by Increasing Depth}

\author{Shao-Qun Zhang\textsuperscript{\rm 1}, Fei Wang\textsuperscript{\rm 2}, \textit{Senior Member, IEEE}, Feng-Lei Fan\textsuperscript{\rm 3*}, \textit{Member, IEEE}% <-this % stops a space
	\thanks{*Dr. F.-L. Fan (hitfanfenglei@gmail.com) is the corresponding author.}% <-this % stops a space
	\thanks{\textsuperscript{\rm 1}Shao-Qun Zhang is with National Key Laboratory for Novel Software Technology, Nanjing University, Nanjing 210023, China. Email: zhangsq@lamda.nju.edu.cn}
	\thanks{\textsuperscript{\rm 2}Dr. Fei Wang is with Department of Population Health Sciences, Weill Cornell Medicine, Cornell University, New York, NY 10065, USA.}
	\thanks{\textsuperscript{\rm 3}Feng-Lei Fan was with Department of Biomedical Engineering, Rensselaer Polytechnic Institute, Troy, NY 12180, USA. Now he is a postdoctoral associate in Department of Population Health Sciences, Weill Cornell Medicine, Cornell University, New York, NY 10065, USA.}
}

% The paper headers
\markboth{Journal of \LaTeX\ Class Files,~Vol.~14, No.~8, August~2021}%
{Shell \MakeLowercase{\textit{et al.}}: A Sample Article Using IEEEtran.cls for IEEE Journals}

% \IEEEpubid{0000--0000/00\$00.00~\copyright~2021 IEEE}
% Remember, if you use this you must call \IEEEpubidadjcol in the second
% column for its text to clear the IEEEpubid mark.
	
\maketitle
	
\begin{abstract}
	Recent years have witnessed an increasing interest in the correspondence between infinitely wide networks and Gaussian processes. Despite the effectiveness and elegance of the current neural network Gaussian process theory, to the best of our knowledge, all the neural network Gaussian processes are essentially induced by increasing width. However, in the era of deep learning, what concerns us more regarding a neural network is its depth as well as how depth impacts the behaviors of a network. Inspired by a width-depth symmetry consideration, we use a shortcut network to show that increasing the depth of a neural network can also give rise to a Gaussian process, which is a valuable addition to the existing theory and contributes to revealing the true picture of deep learning. Beyond the proposed Gaussian process by depth, we theoretically characterize its uniform tightness property and the smallest eigenvalue of the Gaussian process kernel. These characterizations can not only enhance our understanding of the proposed depth-induced Gaussian process but also pave the way for future applications. Lastly, we examine the performance of the proposed Gaussian process by regression experiments on two benchmark data sets.
\end{abstract}
	
	
\section{INTRODUCTION} \label{sec:introduction}
Currently, kernel methods and deep neural networks are two of the most remarkable machine learning methodologies. Recent years have witnessed lots of works on their connection. Lee  \textit{et al.} \cite{lee2017deep} pointed out that randomly initializing parameters of an infinitely wide network gives rise to a Gaussian process, which is referred to as \textit{neural network Gaussian processes} (NNGP). Due to the attraction of this idea, the studies of NNGP have been scaled into more types of networks, such as attention-based models \cite{hron2020infinite} and recurrent networks \cite{yang2019:rnn}.

A Gaussian process is a classical non-parametric model. The equivalence between an infinitely wide fully-connected network and a Gaussian process has been established in \cite{neal1996:priors, lee2017deep}. Given a fully-connected multi-layer network whose parameters are i.i.d. randomly initialized, the output of each neuron is an aggregation of neurons in the preceding layer whose outputs are also i.i.d. When the network width goes infinitely large, according to the Central Limit Theorem \cite{fischer2010history}, the output of each neuron conforms to the Gaussian distribution. As a result, the output function expressed by the network is essentially a Gaussian process. The correspondence between neural networks and Gaussian processes allows the exact Bayesian inference using the neural network \cite{lee2017deep}.  

% One notable achievement accomplished by NTK is the convincing explanation to the question why an over-parameterized network works well in terms of training and generalization in practice []. This problem has been a puzzle over a long period of time. It was previously demonstrated that a deep network can achieve zero training loss with the simple gradient algorithm, which counters the traditional view that training a deep network is a NP-hard non-convex optimization problem. On the other hand, a deep network is an overly parameterized model which often has more parameters than the number of data. But the traditional custom is that a more complicated network tends to not generalize, which cannot elaborate the success of deep networks. 

Despite the achievements of the current NNGP theory, it has an important limit that is not addressed satisfactorily. So far, the neural network Gaussian process is essentially induced by increasing width, regardless of how many layers are stacked in a network. But in the era of deep learning, what concerns us more regarding deep learning is its depth and how the depth affects the behaviors of a neural network, since the depth is the major element accounting for the power of deep learning. Although that the current NNGP theory is beautiful and elegant in its form, unfortunately, it can not accommodate our concern adequately. Therefore, it is highly necessary to expand the scope of the existing theory to include the depth issue. Specifically, our natural curiosity is what is going to happen if we have an infinitely deep but finitely wide network. Can we derive an NNGP by increasing depth rather than width, which contributes to understanding the true picture of deep learning? If this question is positively answered, we are able to reconcile the successes of deep networks and the elegance of the NNGP theory. What's more, as a valuable addition, the depth-induced NNGP greatly enlarges the scope of the existing NNGP theory, which is posited to open lots of doors for research and translation opportunities in this area.

\begin{figure}[t]
	\centering
	\includegraphics[width=\linewidth]{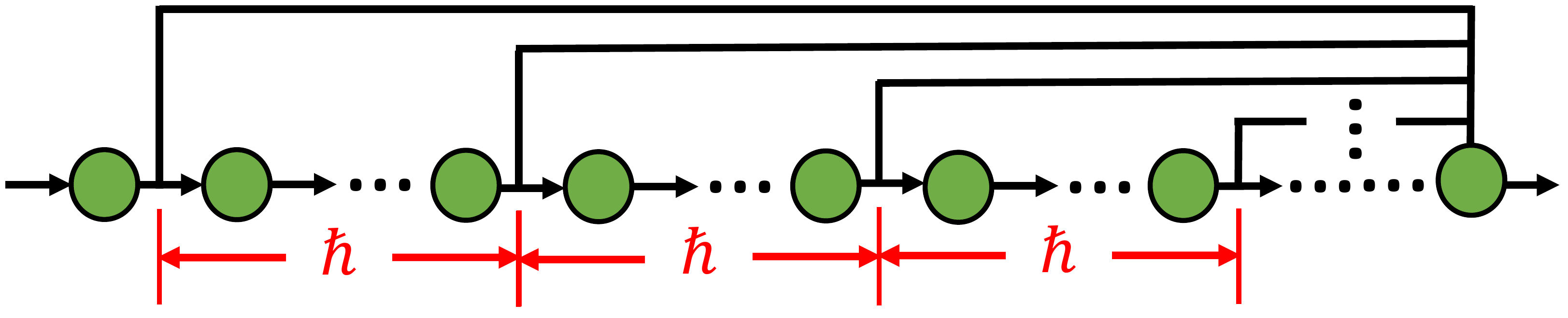}
	\caption{A deep topology that can induce a neural network Gaussian process by increasing depth.}
	\label{Figure_topology}
	\vspace{-0.5cm}
\end{figure}

The above idea is well-motivated based on a width-depth symmetry consideration. Previously, Lu \textit{et al.} \cite{lu2017expressive} and Hornik \textit{et al.} \cite{hornik1989multilayer} have respectively proved that the width-bounded and depth-bounded neural networks are universal approximators. Fan \textit{et al.} \cite{fan2020quasi} suggested that a wide network and a deep network can be converted to each other with a negligible error by De Morgan's law. Since somehow there exists a symmetry between width and depth, deepening a neural network in certain conditions can likely lead to an NNGP as well. Along this direction, we investigate the feasibility of inducing an NNGP by depth (NNGP$^{(d)}$), with a network of a shortcut topology in Figure \ref{Figure_topology}. The characteristic of this topology is that outputs of intermediate layers with a gap of $\hbar$ are aggregated in the final layer, yielding the network output. Such a shortcut topology has been successfully applied to medical imaging \cite{you2019ct} and computer vision \cite{fan2018sparse} as a backbone structure. 

An NNGP by width (NNGP$^{(w)}$) is accomplished by summing the i.i.d. output terms of infinitely many neurons and applying Central Limit Theorem. In contrast, for the topology in Figure \ref{Figure_topology}, as the depth increases, the outputs of increasingly many neurons are aggregated together. We constrain the random weights and biases such that those summed neurons turn weakly dependent by the virtue of their separation. Consequently, when going infinitely deep, the network is also a function drawn from a Gaussian process according to the generalized Central Limit Theorem under weak dependence \cite{b1995:clt}. Beyond the proposed NNGP$^{(d)}$, we theoretically prove that NNGP$^{(d)}$ is uniformly tight and provide a tight bound of the smallest eigenvalue of the concerned NNGP$^{(d)}$ kernel. From the former, one can determine the properties of NNGP$^{(d)}$ such as the functional limit and continuity, while the non-trivial lower and upper bounds mirror the characteristics of the derived kernel, which constitutes a cornerstone for its optimization and generalization properties. 

\textbf{Main Contributions.} In this manuscript, we establish the NNGP by increasing depth, in contrast to the present mainstream NNGPs that are induced by width. Our work substantially enlarges the scope of the existing elegant NNGP theory, making a stride towards understanding the true picture of deep learning. Furthermore, we investigate the essential properties of the proposed NNGP and its associated kernel, which lays a solid foundation for future research and applications. Lastly, we implement an NNGP$^{(d)}$ kernel and apply it for regression experiments on benchmark datasets.

\section{PRELIMINARIES}
Let $[N] = \{1,2,\dots,N\}$ be the set for an integer $N > 0$. Given a function $g(n)$, we denote by $h_1(n)=\Theta(g(n))$ if there exist positive constants $c_1, c_2,$ and $n_0$ such that $c_1g(n) \leq h_1(n) \leq c_2g(n)$ for every $n \geq n_0$; $h_2(n)=\mathcal{O}(g(n))$ if there exist positive constants $c$ and $n_0$ such that $h_2(n) \leq cg(n)$ for every $n \geq n_0$; $h_3(n)=\Omega(g(n))$ if there exist positive constants $c$ and $n_0$ such that $h_3(n) \geq cg(n)$ for every $n \geq n_0$. Let $\| \mathbf{W} \|$ denote the matrix norm for the matrix $\mathbf{W} \in \mathbb{R}^{n \times m}$. Throughout this paper, we employ the maximum spectral norm
\[
\|\mathbf{W}\| \overset{\underset{\mathrm{def}}{}}{=} \max_k |\lambda_k|, \quad\text{for}\quad k \in [\min\{m,n\}],
\]
as the matrix norm \cite{meyer2000:norm}, where $\lambda_k$ denotes the $k$-th singular value of the matrix $\mathbf{W}$. Let $|\cdot|_{\#}$ denote the number of elements, \textit{e.g.}, $|\mathbf{W}|_{\#}=nm$. Finally, we provide several definitions for the characterization of inputs and parameters.
\begin{definition} \label{def:inputs}
	A data distribution $P$ is said to be \textbf{well-scaled}, if the following conditions hold for $\boldsymbol{x} \in \mathbb{R}^d$:
	\begin{enumerate}
		\item $\int \boldsymbol{x} \dif P\left(\boldsymbol{x}\right) = 0 $;
		\item $\int\|\boldsymbol{x}\|_{2} \dif P(\boldsymbol{x})=\Theta(\sqrt{d})$;
		\item $\int\|\boldsymbol{x}\|_{2}^{2} \dif P(\boldsymbol{x})=\Theta(d)$.
	\end{enumerate}
\end{definition}
\begin{definition} \label{def:well_posed}
	A  function $\sigma:\mathbb{R}\to\mathbb{R}$ is said to be \textbf{well-posed}, if $\sigma$ is first-order differentiable, and its derivative is bounded by a certain constant $C_{\sigma}$. Specially, the commonly used activation functions like ReLU, tanh, and sigmoid are well-posed (Please see Table~\ref{tab:activation}).
\end{definition}
\begin{definition} \label{def:stable_pertinent}
	A matrix $\mathbf{V}$ is said to be \textbf{stable-pertinent} for a well-posed activation function $\sigma$, in short $\mathbf{V} \in SP(\sigma)$, if the inequality $C_{\sigma} \|\mathbf{V}\| < 1$ holds.
\end{definition}

\begin{table}[t]
	\centering
	\caption{Well-posedness of the commonly-used activation functions.}
	\label{tab:activation}
	\begin{tabular}{|l|l|}
		\hline
		Activations & Well-Posedness  \\ \hline
		ReLU & $\|\sigma'(\boldsymbol{x})\| \leq 1$ \\ \hline
		$\tanh$     & $\|\sigma'(\boldsymbol{x})\| = \| 1- \sigma^2(\boldsymbol{x}) \| \leq 1$  \\ \hline
		sigmoid  & $\|\sigma'(\boldsymbol{x})\| = \| \sigma(\boldsymbol{x})(1- \sigma(\boldsymbol{x})) \| \leq 1/4$  \\ \hline
	\end{tabular} 
	\vspace{-0.2cm}
\end{table}

\section{MAIN RESULTS} \label{sec:GP}
In this section, we formally present the neural network Gaussian process NNGP$^{(d)}$, led by an infinitely deep but finitely wide neural network with i.i.d. weight parameters. We also derive the uniform tightness for NNGP$^{(d)}$ with the increased depth and the bound estimation of its associated kernel's smallest eigenvalue. These two valuable characterizations serve as the solid cornerstones for NNGP$^{(d)}$.

\subsection{Neural Network Gaussian Process with Increasing Depth} \label{subsec:GP}
Consider an $L$-layer neural network whose topology is illustrated as Figure~\ref{Figure_topology}, the feed-forward propagation follows
\begin{equation} \label{eq:forward}
	\begin{cases}
		& \boldsymbol{z}^0 = \boldsymbol{x} \\
		& \boldsymbol{z}^l = \sigma(\mathbf{W}^l \boldsymbol{z}^{l-1} + \boldsymbol{b}^l) \ ,
	\end{cases}
\end{equation}
where $\mathbf{W}^l$ and $\boldsymbol{b}^l$ are the weight matrix and bias vector of the $l^{th}$ layer, respectively, and $\sigma$ is the activation function. Invoking shortcut connections, the final output of this network is a mean of $\kappa \in \mathbb{N}^+$ previous layers with an equal separation $\hbar\in\mathbb{N}^+$ and $l_1 \in [L]$
\begin{equation} \label{eq:shortcut}
	f(\boldsymbol{x}; \boldsymbol{\theta}) = \frac{1}{\sqrt{M_{\boldsymbol{z}}}} \sum_{\kappa=0}^{K} \mathbf{1}^{l_1 + \kappa \hbar} \boldsymbol{z}^{l_1 + \kappa \hbar} \ ,
\end{equation}
where the matrix $\mathbf{1}^{l_1+\kappa\hbar} \in \{1\}^{n_o \times n_{l_1+\kappa\hbar}}$ indicates the unit shortcut connection between $\boldsymbol{z}^{l_1+\kappa\hbar}$ and the final layer, and $M_{\boldsymbol{z}}$ denotes the summed number of concerned hidden neurons
\[
M_{\boldsymbol{z}} = \sum\nolimits_{\kappa=0}^{K} n_{\kappa} \quad\text{with}\quad n_{\kappa} = |\boldsymbol{z}^{l_1 + \kappa \hbar}|_{\#} \ .
\]
Let $\boldsymbol{\theta} = \mathrm{concat}( \bigcup_{l=1}^L \mathrm{vec}(\boldsymbol{b}^l,\mathbf{W}^l) )$ be the concatenation of all vectorized weight matrices and $n = |\boldsymbol{\theta}|_{\#}$. Regarding the neural network $f:\mathbb{R}^d \to \mathbb{R}^{n_o}$, we present the first main theorem as follows:
\begin{theorem} \label{thm:existing}
	The infinitely deep neural network, defined by Eqs.~\eqref{eq:forward} and~\eqref{eq:shortcut}, is equivalent to a Gaussian process NNGP$^{(d)}$, if $\sigma$ is well-posed and the augmented parameter matrix of each layer is stable-pertinent for $\sigma$, that is, $( \mathbf{W}^l, \boldsymbol{b}^l ) \in SP(\sigma)$, for $\forall~ l\in [L]$.
\end{theorem}
Theorem~\ref{thm:existing} states that our proposed neural network converges to a Gaussian process as $L \to \infty$. Given a data set $\mathcal{D}=\{(\boldsymbol{x}_i,y_i)\}_{i=1}^N$, the limit output variables of this network belongs to a multivariate Gaussian distribution $\mathcal{N}(0,\mathbf{K}_{\mathcal{D},\mathcal{D}})$ whose mean equals to 0 and covariance matrix is an $N \times N$ matrix, the $(i,j)$-entry of which is defined as
\begin{equation}
	\mathbf{K}(\boldsymbol{x}_i,\boldsymbol{x}_j) = \mathbb{E}[\langle f(\boldsymbol{x}_i; \boldsymbol{\theta}), f(\boldsymbol{x}_j; \boldsymbol{\theta})\rangle], \quad\text{for}\quad \boldsymbol{x}_i, \boldsymbol{x}_j \in \mathcal{D} \ .
	\label{eqn:kernel}    
\end{equation}

The key idea of proving Theorem~\ref{thm:existing} is to show that our proposed neural network converges to a Gaussian process as depth increases according to the generalized Central Limit Theorem with weakly dependent variables instead of random ones. To implement this idea, we constrain the weights and biases to enable that random variables of two hidden layers with a sufficient separation degenerate to weak dependence, \textit{i.e.}, mixing processes. By aggregating the weakly dependent variables to the final layer via shortcut connections, the output of the proposed network converges to a Gaussian process as the depth goes to infinity. The key steps are formally stated by Lemmas~\ref{lemma:weak} and~\ref{lemma:CLT} as follows:

% In contrast with conventional NNGP$^{w}$ that converges to a Gaussian process by Central Limit Theorem with random variables, 
\begin{lemma} \label{lemma:weak}
	Provided a well-posed $\sigma$ and stable-pertinent parameter matrices, the concerned neural network comprises a stochastic sequence of weakly dependent variables as the depth goes to infinity.
\end{lemma}
\begin{proof}
	Let $\boldsymbol{\mathcal{H}}_s^t$ denote the distribution of the random variable sequence $\{ Z^s, Z^{s+1}, \dots, Z^t \}$, where $0 \leq s < t$, and $\mathbf{Z}^{-t} = (Z^0, \dots, Z^t)$ indicates the vector of random variables before the timestamp $t$. We define a coefficient~\cite{zhang2020:hrp} as 
	\[
	\beta(s) = \sup_t \mathbb{E}_{ \mathbf{Z}^{-t}} \left[~ \| \boldsymbol{\mathcal{H}}_{t+s}^{+\infty}( \cdot\mid\mathbf{Z}^{-t} ) - \boldsymbol{\mathcal{H}}_{t+s}^{+\infty}(\cdot) \|_{\mu} \right] \ ,
	\]
	where $\boldsymbol{\mathcal{H}}(\cdot|\cdot)$ stands for a conditional probability distribution, and $\mu$ denotes a probability measure, or equally the $\sigma$-algebra of events $\mathcal{G}$ \cite{joe1997:Sklar}, which satisfies
	\[
	\| P-Q \|_{\mu} = \sup_{z \in \mathcal{G}} | P(z) - Q(z) | \ ,
	\]
	for two probability distributions $P$ and $Q$. According to Eq.~\eqref{eq:forward}, we have $Z^l = \sigma(\tilde{\mathbf{W}}^l \tilde{Z}^{l-1})$ for all $l \in [L]$, where $\tilde{\mathbf{W}}^l = ( \mathbf{W}^l, \boldsymbol{b}^l )$ and $\tilde{Z}^{l-1}= (Z^{l-1};1)$. Given the well-posed $\sigma$ and stable-pertinent parameter matrices, \textit{i.e.}, $\tilde{\mathbf{W}}^l \in SP(\sigma) $ for any $l \in [L]$, the followings hold
	\[
	\frac{\partial Z^{l+s}}{\partial Z^l} \leq R^s \quad\text{and}\quad \mathbb{E}_{Z^l,\tilde{\mathbf{W}}} \left[ \frac{|Z^{l+s}Z^l|}{|Z^l|} \right] \leq R^s |Z^l| \ ,
	\]
	where $C_{\sigma} \| \tilde{\mathbf{W}}^l \| \leq R <1$ and $s\in\mathbb{N}^+$. This implies that (informally) the ``dependence'' between variables $Z^l$ and $Z^{l+s}$ goes to be weak as $s \to \infty$. From Sklar’s theorem, we have
	\[
	\mathcal{H}^{l+s}(\cdot) \wedge \mathcal{H}^{l}(\cdot) = C_{l}(s) \cdot \mathcal{H}^{l+s}(\cdot) \cdot \mathcal{H}^{l}(\cdot) \ ,
	\]
	where $C_l(s) \in \Omega(R^s)$ is the corresponding Copula function. Further, it holds
	\[
	\boldsymbol{\mathcal{H}}_{l+s}^{+\infty}( \cdot\mid\mathbf{Z}^{-l} ) - \boldsymbol{\mathcal{H}}_{l+s}^{+\infty}(\cdot) = \sum_{l} C_l(s) \cdot \boldsymbol{\mathcal{H}}_{l+s}^{+\infty}(\cdot) \cdot \mathcal{H}^{l}(\cdot) \ .
	\]
	Since $C_l(s)$ is independent to the layer (\emph{i.e., time}) index $l$, we assert that $\beta(s)$ is proportional to $C_l(s)$. Thus, we have
	\[
	\beta(s) \rightarrow 0 \quad \text{as} \quad s \rightarrow +\infty \ .
	\]
	Therefore, the sequence $\{Z^t\}$ led by Eq.~\eqref{eq:forward} is $\beta$-mixing, or equally weakly dependent, which completes the proof.
\end{proof}

\begin{lemma} \label{lemma:CLT}
	Suppose that (i) a random variable sequence $\{ Z^s \}_{s=1}^{t}$ is weakly independent, satisfying $\beta$-mixing with an exponential convergence rate, (ii) for $ \forall s \in [t]$, we have
	\[
	\mathbb{E}[Z^s] = 0 \quad\text{and}\quad \mathbb{E}[(Z^s)^2] < \infty .
	\]
	Let $\Lambda_t = Z^1 + Z^2 + \dots + Z^t$, then we have
	\[
	\mu \overset{\underset{\mathrm{def}}{}}{=} \lim\limits_{t\to\infty} \mathbb{E}[\Lambda_t]=0 \quad\text{and}\quad \upsilon^2 \overset{\underset{\mathrm{def}}{}}{=} \lim\limits_{t\to\infty} \mathbb{E}(\Lambda_t^2) / t < \infty \ .
	\]
	Further, the limit variable $\Lambda_t/(\upsilon\sqrt{t})$ converges in distribution to $\mathcal{N}(0,1)$ as $t\to\infty$, provided $\upsilon \neq 0$.
\end{lemma}
Lemma~\ref{lemma:CLT} is a variant of the generalized Central Limit Theorem under weak dependence. The proof idea can be summarized as follows. From \cite{doukhan2012:mixing}, it's observed that an $\beta$-mixing sequence with an exponential convergence rate can be covered by the $\alpha$-mixing one with $\mathcal{O}(t^{-5})$. Thus, the conditions of Lemma~\ref{lemma:CLT} satisfy the preconditions of the generalized Central Limit Theorem under weak dependence~\cite[Theorem 27.5]{b1995:clt}. This lemma also has alternative proofs according to the encyclopedic treatment of limit theorems under mixing conditions. Interested readers can refer to \cite{bradley2007:mixing} for more details.

\vspace{0.2cm}
\noindent\emph{\textbf{Finishing the Proof of Theorem~\ref{thm:existing}}}. Let $\boldsymbol{z}^l$ denote the output variables of the $l$-th layer, which satisfies that $ \boldsymbol{z}^{l+1} = \sigma( \mathbf{W}^{l+1} \boldsymbol{z}^l + \boldsymbol{b}^{l+1} )$ and $\boldsymbol{z}^{0} = \boldsymbol{x}$. Because the weights and biases are taken to be i.i.d., the sequence $\{\boldsymbol{z}^l\}$ $(l\in[L])$ leads to a stochastic process, and the post-activations in the same layer, such as $\boldsymbol{z}_i^l$ and $\boldsymbol{z}_j^l$ are independent for $i \neq j$. Given an integer $\hbar \in \mathbb{N}^+$, we select a sub-sequence of $\{ \boldsymbol{z}^l \}$ as follows:
\[
\mathcal{Z}^{l_1}_{\hbar} = \{ \boldsymbol{z}^{l_1+\hbar}, \boldsymbol{z}^{l_1+2\hbar}, \dots, \boldsymbol{z}^{l_1+\kappa\hbar}, \dots \} \ ,
\]
for $l_1\in [L]$ and $\kappa \in \mathbb{N}^+$, which satisfies ${l_1+\kappa\hbar} \leq L$. From Lemma~\ref{lemma:weak}, the sequence $\mathcal{Z}^{l_1}_{\hbar}$ leads to a weakly dependent stochastic process. Aggregating this sub-sequence with $\kappa$ shortcut connections to the output layer, the output of the concerned neural network converges to a Gaussian process as $\kappa\to\infty$ as well as $L \to \infty$, from Lemma~\ref{lemma:CLT}. $\hfill\square$

\begin{figure}[hbtp]
	\centering
	\includegraphics[width=0.7\linewidth]{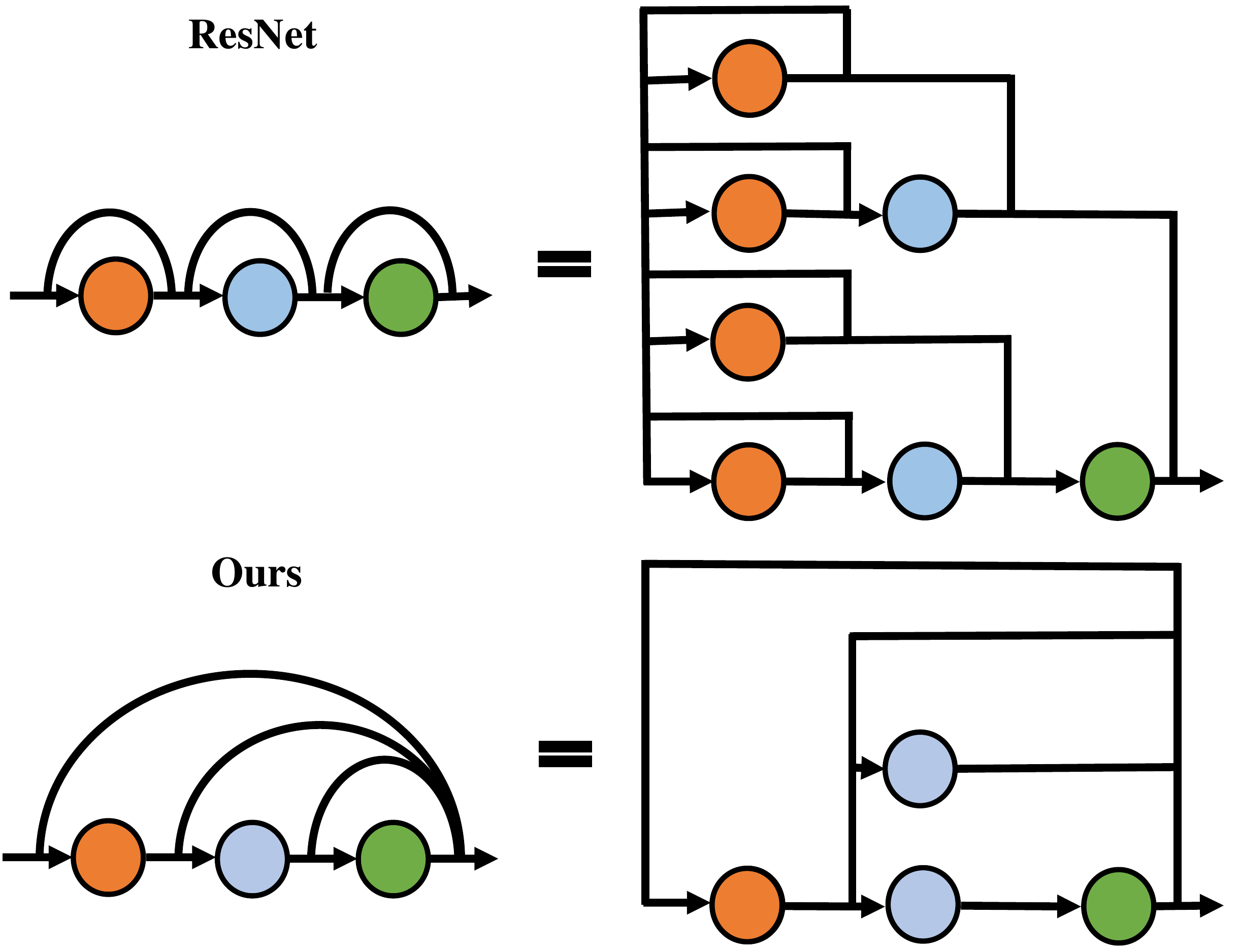}
	\caption{Both ResNet and ours can be regarded as wide networks in the unraveled view.}
	\label{Figure_unravel}
\end{figure}

\noindent\textbf{Discussions}. To the best of our knowledge, our proposed NNGP$^{(d)}$ is the first NNGP induced by increasing depth. Currently, there is no rigorous definition for width and depth. The way we claim depth just aligns with the conventional usage of the width and depth for a neural network, in which the depth is understood as the maximum number of neurons among all possible routes from the input to the output, and the width is the maximum number of neurons in a layer. As illustrated in Figure \ref{Figure_unravel}, if examined in an unraveled view, our network is a simultaneously wide and deep network due to the layer reuse in different routes. However, we argue that this will not affect our claim because not every layer has an infinite width in the unraveled view, which is different from the key character of NNGP$^{(w)}$. What's more, the conventional usage is more acceptable relative to the unraveled view; otherwise, it is against common sense because the ResNet is also a wide network in the unraveled view.

% \ff{Besides, it is worthy of including some technical details. First, the existence of our proposed NNGP$^{(d)}$ relies heavily on the generalized Center Limit Theorem, which holds on three conditions as shown in Lemma~\ref{lemma:CLT}: i) The random variable sequence is weakly dependent; ii) the random variable  maintains finite mathematical variance; iii) the input data is drawn from a compact set. Second, the parameter $\hbar$ is no need an equal separation. As shown in Lemma~\ref{lemma:weak}, this parameter provides a separation of the network depth to ensure that the layers at both ends of the separation interval are weakly independent. }

The existence of the proposed NNGP$^{(d)}$ kernel relies heavily on the generalized Central Limit Theorem, which holds on three conditions as mentioned in Lemma~\ref{lemma:CLT}: i) The random variable sequence is weakly dependent; ii) the random variable maintains a finite mathematical variance; iii) the input data are drawn from a compact set. According to these conditions, we make two remarks. First, as shown in Lemma~\ref{lemma:weak}, $\hbar$ provides a separation of the network depth to ensure that the layers at both ends of the separation interval are weakly dependent. Therefore, $\hbar$ is not necessarily an equal separation. Second, our proof doesn't prescribe the distribution of the input data, as long as the input data are drawn from a compact set.

\subsection{Uniform Tightness of NNGP$^{(d)}$}
In this subsection, we delineate the asymptotic behavior of NNGP$^{(d)}$ as the depth goes to infinity. Here, we assume that the weights and biases are i.i.d. sampled from $\mathcal{N}(0,\eta^2)$. Per the conditions of Theorem~\ref{thm:existing}, we have the following theorem:
\begin{theorem} \label{thm:asymptotic}
	For any $l_1 \in [L]$, the stochastic process, described in Lemma~\ref{lemma:weak}, is \textbf{uniformly tight} in $\mathcal{C}(\mathbb{R}^d,\mathbb{R})$.
\end{theorem}
Theorem~\ref{thm:asymptotic} reveals that the stochastic process contained by our network (illustrated in Figure \ref{Figure_topology}) is uniformly tight, which is an intrinsic characteristic of NNGP$^{(d)}$. Based on Theorem~\ref{thm:asymptotic}, one can obtain not only the functional limit and continuity properties of NNGP$^{(d)}$, in analogy to the results of NNGP$^{(w)}$ \cite{bracale2020:asymptotic}.
Similarly, we start the proof of Theorem~\ref{thm:asymptotic} with some useful lemmas.

\begin{lemma} \label{lemma:tightness}
	Let $\{Z^1, Z^2, \dots, Z^t\}$ denote a sequence of random variables in $\mathcal{C}(\mathbb{R}^d,\mathbb{R})$. This stochastic process is \textbf{uniformly tight} in $\mathcal{C}(\mathbb{R}^d,\mathbb{R})$, if (1) $\boldsymbol{x}=\boldsymbol{0}$ is a uniformly tight point of $Z^s(\boldsymbol{x})$ ($s \in [t]$) in $\mathcal{C}(\mathbb{R}^d,\mathbb{R})$; (2) for any $\boldsymbol{x}, \boldsymbol{x}' \in \mathbb{R}^d$ and $s \in [t]$, there exist $\alpha, \beta, C >0$, such that $\mathbb{E} \left[ | Z^s(\boldsymbol{x}) - Z^s(\boldsymbol{x}') |^{\alpha} \right] \leq C \| \boldsymbol{x} - \boldsymbol{x}' \|^{\beta+d} $.
\end{lemma}
Lemma~\ref{lemma:tightness} is the core guidance for proving Theorem~\ref{thm:asymptotic}. This lemma can be straightforwardly derived from Kolmogorov Continuity Theorem~\cite{stroock1997:Kolmogorov}, provided the Polish space $(\mathbb{R}, |\cdot|)$. 
\begin{lemma} \label{lemma:1}
	Based on the notations of Lemma~\ref{lemma:tightness}, $\boldsymbol{x}=\boldsymbol{0}$ is a uniformly tight point of $Z^s(\boldsymbol{x})$ ($s \in [t]$) in $\mathcal{C}(\mathbb{R}^d,\mathbb{R})$.
\end{lemma}
\begin{proof}
	It suffices to prove that 1) $\boldsymbol{x}=\boldsymbol{0}$ is a tight point of $Z^s(\boldsymbol{x})$ ($s \in [t]$) in $\mathcal{C}(\mathbb{R}^d,\mathbb{R})$ and 2) the statistic $(Z^1(\boldsymbol{0})+ \dots + Z^s(\boldsymbol{0})) / s$ converges in distribution as $s \to \infty$. Note that 1) is self-evident since every probability measure in $(\mathbb{R}, |\cdot|)$ is tight~\cite{zhang2021:arise}; 2) has been proved by Theorem~\ref{thm:existing}. Therefore, we finish the proof of this lemma.
\end{proof}

\noindent\textbf{Remark.} Notice that the convergence in distribution ($\overset{\underset{\mathrm{d}}{}}{\to}$) from Lemmas~\ref{lemma:CLT} and~\ref{lemma:1} paves the way for the convergence of expectations. Specifically, provided a linear and bounded functional $\mathcal{F}: \mathcal{C}(\mathbb{R}^d;\mathbb{R}^{n^*}) \to \mathbb{R}$ as $L\to\infty$ and a function $f$ which satisfies that $f(\boldsymbol{x};\boldsymbol{\theta}) \overset{\underset{\mathrm{d}}{}}{\to} f^*$, then we have
$\mathcal{F} (f(\boldsymbol{x};\boldsymbol{\theta})) \overset{\underset{\mathrm{d}}{}}{\to} \mathcal{F}(f^*)$
and $\mathbb{E} \left[ \mathcal{F} (f(\boldsymbol{x};\boldsymbol{\theta})) \right] \to \mathbb{E} \left[ \mathcal{F}(f^*) \right]$ according to General Transformation Theorem \cite[Theorem 2.3]{van2000:asymptotic} and Uniform Integrability \cite{billingsley2013:convergence}, respectively. These results may serve as solid bases for development and applications of NNGP$^{(d)}$ in the future.

\begin{lemma} \label{lemma:2}
	Based on the notations of Lemma~\ref{lemma:tightness}, for any $\boldsymbol{x}, \boldsymbol{x}' \in \mathbb{R}^d$ and $s \in [t]$, there exist $\alpha, \beta, C >0$, such that 
	\[
	\mathbb{E} \left[ \sup_i \big| Z^s_i(\boldsymbol{x}) - Z^s_i(\boldsymbol{x}') \big|^{\alpha} \right] \leq C \| \boldsymbol{x} - \boldsymbol{x}' \|^{\beta+d} \ .
	\]
\end{lemma}

The proof of Lemma~\ref{lemma:2} can be accessed from Appendix A. Further, Theorem~\ref{thm:asymptotic} can be completely proved by invoking Lemmas~\ref{lemma:1} and~\ref{lemma:2} into Lemma~\ref{lemma:tightness}.

\subsection{Tight Bound for the Smallest Eigenvalue}
In this subsection, we provide a tight bound for the smallest eigenvalue of the NNGP$^{(d)}$ kernel. For the NNGP$^{(d)}$ with ReLU activation, we have the following theorem:
\begin{theorem} \label{thm:smallest}
	Suppose that $\boldsymbol{x}_1, \dots, \boldsymbol{x}_N$ are i.i.d. sampled from $P_X = \mathcal{N}(0,\eta^2)$, and $P_X$ is a well-scaled distribution, then for an integer $r \geq 2$, with probability $1-\delta>0$, we have $\lambda_{\min} \left( \mathbf{K}_{\mathcal{D},\mathcal{D}} \right) = \Theta(d)$, where
	\[
	\delta \leq N e^{-\Omega(d)} + N^{2} e^{-\Omega\left(d N^{-2 /(r-0.5)}\right)} \ .
	\]
\end{theorem}
Theorem~\ref{thm:smallest} provides a tight bound for the smallest eigenvalue of the NNGP$^{(d)}$ kernel. This nontrivial estimation mirrors the characteristics of this kernel, and usually be used as a key assumption for optimization and generalization. 

The key idea of proving Theorem~\ref{thm:smallest} is based on the following inequalities about the smallest eigenvalue of real-valued symmetric square matrices. Given two symmetric matrices $\mathbf{P}, \mathbf{Q}\in\mathbb{R}^{m \times m}$, it's observed that
\begin{equation} \label{eq:lamda}
	\left\{\begin{aligned}
		&\lambda_{\min}(\mathbf{P}\mathbf{Q}) \geq \lambda_{\min}(\mathbf{P}) \min_{i\in[m]} \mathbf{Q}(i,i) \ , \\
		&\lambda_{\min}(\mathbf{P}+\mathbf{Q}) \geq \lambda_{\min}(\mathbf{P}) + \lambda_{\min}(\mathbf{Q}) \ .
	\end{aligned} \right.
\end{equation}
From Eqs.~\eqref{eq:shortcut} and~\eqref{eqn:kernel}, we can unfold $\mathbf{K}(\boldsymbol{x}_i, \boldsymbol{x}_j)$ as a sum of covariance of the sequence of random variables $\{\boldsymbol{z}^{l_1+\kappa\hbar}\}$. Thus, we can bound $\lambda_{\min} \left( \mathbf{K}_{\mathcal{D},\mathcal{D}} \right)$ by $\mathrm{Cov}(\boldsymbol{z}^{l_1},\boldsymbol{z}^{l_1})$ via a chain of feedforward compositions in Eq.~\eqref{eq:forward}. For conciseness, we put the proof of Theorem \ref{thm:smallest} into Appendix B.

\section{EXPERIMENTS} \label{sec:experiments}
Generally, the depth can endow a network with a more powerful representation ability than the width. However, it is unclear whether or not the superiority of depth can sustain in the setting of NNGP, as all parameters are random rather than trained. In other words, it is unclear whether our established NNGP$^{(d)}$ is more expressive than NNGP$^{(w)}$. To answer this question, in this section, we apply the NNGP$^{(d)}$ kernel into the generic regression task and then compare its performance on the Fashion-MNIST (FMNIST) and CIFAR10 data sets with that of NNGP$^{(w)}$.

\textbf{NNGP$^{(d)}$ regression.} Provided the data set $\mathcal{D}=\{(\boldsymbol{x}_i,y_i)\}_{i=1}^N$, where $\boldsymbol{x}_i \in \mathbb{R}^{d \times 1}$ is the input, and $y_i \in \mathbb{R}$ is the corresponding label, our goal is to predict $y^*$ for the test sample $\boldsymbol{x}^*$. From Theorem~\ref{thm:existing}, $\boldsymbol{x}_i$ and $\boldsymbol{x}^*$ belong to a multivariate Gaussian process $\mathcal{N}(0,\mathbf{K}^*)$, whose mean equals to 0, and covariance matrix has the following form:
\begin{equation} \label{eq:kernel}
	\mathbf{K}^* = 
	\begin{bmatrix}
		\mathbf{K}_{\mathcal{D},\mathcal{D}} &  \mathbf{K}_{\boldsymbol{x}^*, \mathcal{D}}^{\top}\\
		\mathbf{K}_{\boldsymbol{x}^*, \mathcal{D}} &  \mathbf{K}_{\boldsymbol{x}^*,\boldsymbol{x}^*} \\
	\end{bmatrix} \ ,
\end{equation}
where $\mathbf{K}_{\mathcal{D},\mathcal{D}}$ is an $N \times N$ matrix computed by Eq.~\eqref{eqn:kernel}, and the $i$-th element of $\mathbf{K}_{\boldsymbol{x}^*, \mathcal{D}} \in \mathbb{R}^{1 \times N}$ is $\mathbf{K}(\boldsymbol{x}^*, \boldsymbol{x}_i)$ for $\boldsymbol{x}_i \in \mathcal{D}$. It's observed that Eq.~\eqref{eq:kernel} provides a division paradigm corresponding to the training set and test sample, respectively. Thus, we have $(\cdot\mid\mathcal{D},\boldsymbol{x}^*) \in \mathcal{N}(\mu^*,K^*)$ with
\begin{equation} \label{prediction}
	\left\{\begin{aligned}
		& \mu^* = \mathbf{K}_{\boldsymbol{x}^*, \mathcal{D}} \mathbf{K}_{\mathcal{D},\mathcal{D}} \boldsymbol{y}^{\top} ,\\
		& K^* = \mathbf{K}_{\boldsymbol{x}^*,\boldsymbol{x}^*} - \mathbf{K}_{\boldsymbol{x}^*, \mathcal{D}} \mathbf{K}_{\mathcal{D},\mathcal{D}}^{-1} \mathbf{K}_{\boldsymbol{x}^*, \mathcal{D}}^{\top} \ ,
	\end{aligned} \right.    
\end{equation}
where $\boldsymbol{y} = (y_1,y_2,\dots,y_n)$ denotes the label vector. When the observations are corrupted by the Gaussian additive noise of $\mathcal{N}(0, \eta^2)$, Eq.~\eqref{prediction} becomes
\begin{equation} \label{noised_prediction} 
	\left\{ \begin{aligned}
		& \mu^* = \mathbf{K}_{\boldsymbol{x}^*, \mathcal{D}} (\mathbf{K}_{\mathcal{D},\mathcal{D}}+\eta^2 \mathbf{I}_n) \boldsymbol{y}^{\top} \ ,\\
		& K^* = \mathbf{K}_{\boldsymbol{x}^*,\boldsymbol{x}^*} - \mathbf{K}_{\boldsymbol{x}^*, \mathcal{D}} (\mathbf{K}_{\mathcal{D},\mathcal{D}}+\eta^2 \mathbf{I}_n)^{-1} \mathbf{K}_{\boldsymbol{x}^*, \mathcal{D}}^{\top} \ ,
	\end{aligned} \right.   
\end{equation}
where $\mathbf{I}_n$ is the $n\times n$ identity matrix. For numerical implementation, we calculate the kernels as, for $\boldsymbol{x}_i, \boldsymbol{x}_j \in \mathcal{D}$, 
\begin{equation}
	\mathbf{K}(\boldsymbol{x}_i,\boldsymbol{x}_j) = \mathbb{E}[\langle g(\boldsymbol{x}_i; \boldsymbol{\theta}), g(\boldsymbol{x}_j; \boldsymbol{\theta})\rangle] \ , 
	\label{eqn:bothkernel}    
\end{equation}
where $g(\cdot; \boldsymbol{\theta})$ indicates the deep network or wide network. 

\textbf{Experimental setups.} We conduct regression experiments on FMNIST and CIFAR10 data sets. We respectively sample 1k, 2k, and 3k data from the training sets to construct two kernels and then test the performance of kernels on the test sets. Here, we employ a one-hidden-layer wide network to compute the NNGP$^{(w)}$ kernel, whereas the width of the deep network is set to the number of classes which is the smallest possible width for prediction tasks. For a fair comparison, the depth of NNGP$^{(d)}$ and the width of NNGP$^{(w)}$ are equally set to $200$ ($\hbar=1$).
For classification tasks, the class labels are encoded into an opposite regression formation, where incorrect classes are $-0.1$ and the correct class is $0.9$~\cite{lee2017deep}. For two networks, we employ $\tanh$ as the activation function. Following the setting of NNGP$^{(w)}$\cite{lee2017deep}, all weights are initialized with a Gaussian distribution of the mean $0$ and the variance of $0.3/n_{l}$ for normalization in each layer, where $n_{l}$ is the number of neurons in the $l$-th layer. 
The initialization is repeated $200$ times to compute the empirical statistics of the NNGP$^{(d)}$ and NNGP$^{(w)}$ based on Eq. \eqref{eqn:bothkernel}. We also run each experiment $5$ times for counting the mean and variance of accuracy. All experiments are conducted on Intel Core-i7-6500U.

\begin{table}[t]
	\caption{Test accuracy of regression experiments based on NNGP$^{(d)}$ and NNGP$^{(w)}$ kernels.}
	\begin{center}
		\resizebox{\linewidth}{13mm}{
			\begin{tabular}{c|c|c|c|c}
				\hline
				Model & FMNIST	     & Test accuracy & CIFAR10 & Test accuracy\\
				\hline
				\hline
				NNGP$^{(d)}$  & \multirow{2}{*}{1k} & \textbf{0.345$\pm$0.016} & \multirow{2}{*}{1k} & 0.166$\pm$0.018 \\
				NNGP$^{(w)}$ & & 0.342$\pm$0.021  & & \textbf{0.187$\pm$0.018} \\
				\hline
				NNGP$^{(d)}$ & \multirow{2}{*}{2k}   & 0.352$\pm$0.019 & \multirow{2}{*}{2k} & 0.178$\pm$0.007\\
				NNGP$^{(w)}$ &   &  \textbf{0.373$\pm$0.030} & & \textbf{0.188$\pm$0.012}  \\
				\hline
				NNGP$^{(d)}$  & \multirow{2}{*}{3k} & \textbf{0.372 $\pm$0.024}  & \multirow{2}{*}{3k} & 0.182$\pm$0.005 \\
				NNGP$^{(w)}$ & & 0.365$\pm$0.007 & & \textbf{0.185$\pm$0.019} \\
				\hline
			\end{tabular}
		}
	\end{center}
	\label{tab:nngp_exp}
	\vspace{-0.6cm}
\end{table}

\textbf{Results.} Table~\ref{tab:nngp_exp} lists the performance of the regression experimental results using NNGP$^{(d)}$ and NNGP$^{(w)}$ kernels. It is observed that the test accuracy of NNGP$^{(d)}$ and NNGP$^{(w)}$ kernels are comparable to each other, which implies that NNGP$^{(d)}$ and NNGP$^{(w)}$ kernels are similar to each other in representation ability. The reason may be that both NNGP$^{(d)}$ and NNGP$^{(w)}$ kernels are not stacked kernels. Their difference is mainly the aggregation of independent or weakly dependent variables. Thus, their ability should be similar~\cite{lee2017deep}.

%We think the reason is as follows. The wide network here is a one-hidden-layer network whose NNGP kernel is not the stacked kernel. Thus, given a sample $\boldsymbol{x}_i$, the output of the wide network is the mean of i.i.d. terms directly dependent on $\boldsymbol{x}_i$. In contrast, the output of the deep network averages the weakly dependent terms directly on $\boldsymbol{x}_i$ as well. Thus, both NNGP$^{(d)}$ and NNGP$^{(w)}$ are limiting aggregations of hidden variables directly connecting $\boldsymbol{x}_i$, and their discriminative ability should be similar~\citep{lee2017deep}. 

Next, we use the angular plot to investigate how the separation $\hbar$ affects the representation ability of the NNGP$^{(d)}$ kernel. The angle is computed according to
\begin{equation*}
	\alpha = \arccos \left( \frac{\mathbf{K}(\boldsymbol{x}_1,\boldsymbol{x}_2)}{\sqrt{\mathbf{K}(\boldsymbol{x}_1,\boldsymbol{x}_1 ) \cdot \mathbf{K}(\boldsymbol{x}_2,\boldsymbol{x}_2)}} \right),
\end{equation*} 
and the angular plot manifests the relationship between kernel values and angles. If an angular plot comes near zero, the kernel cannot well recognize the difference between samples. Otherwise, the kernel is regarded to have a better discriminative ability. We set the network depth to $200\times\hbar$ so that the NNGP$^{(d)}$ kernel is empirically computed by aggregating $\kappa=200$ shortcut connections with a separation of $\hbar$ between neighboring shortcut connections. Figure~\ref{Figure_angular_form} illustrates the angularities of  NNGP$^{(d)}$ kernels with $\hbar = 1,3$ for FMNIST-1k training data. It is observed that the angular plot of the kernel with $\hbar = 3$ is compressed to be closer to zero relative to that of the kernel with $\hbar = 1$, which implies that a smaller separation $\hbar$ may induce a powerful NNGP$^{(d)}$ kernel.

To have a better understanding of the proposed NNGP$^{(d)}$ kernel, we explore the impacts of the separation $\hbar$, the number of samples, the parameter variance, and the network size on it, as well as the computation time of the kernel in Appendix C. We have shared all our code in \href{https://github.com/FengleiFan/NNGP_by_Depth}{link1} and \href{http://www.lamda.nju.edu.cn/zhangsq/}{link2}.

\begin{figure}[!htb]
	\centering
	\includegraphics[width=0.6\linewidth]{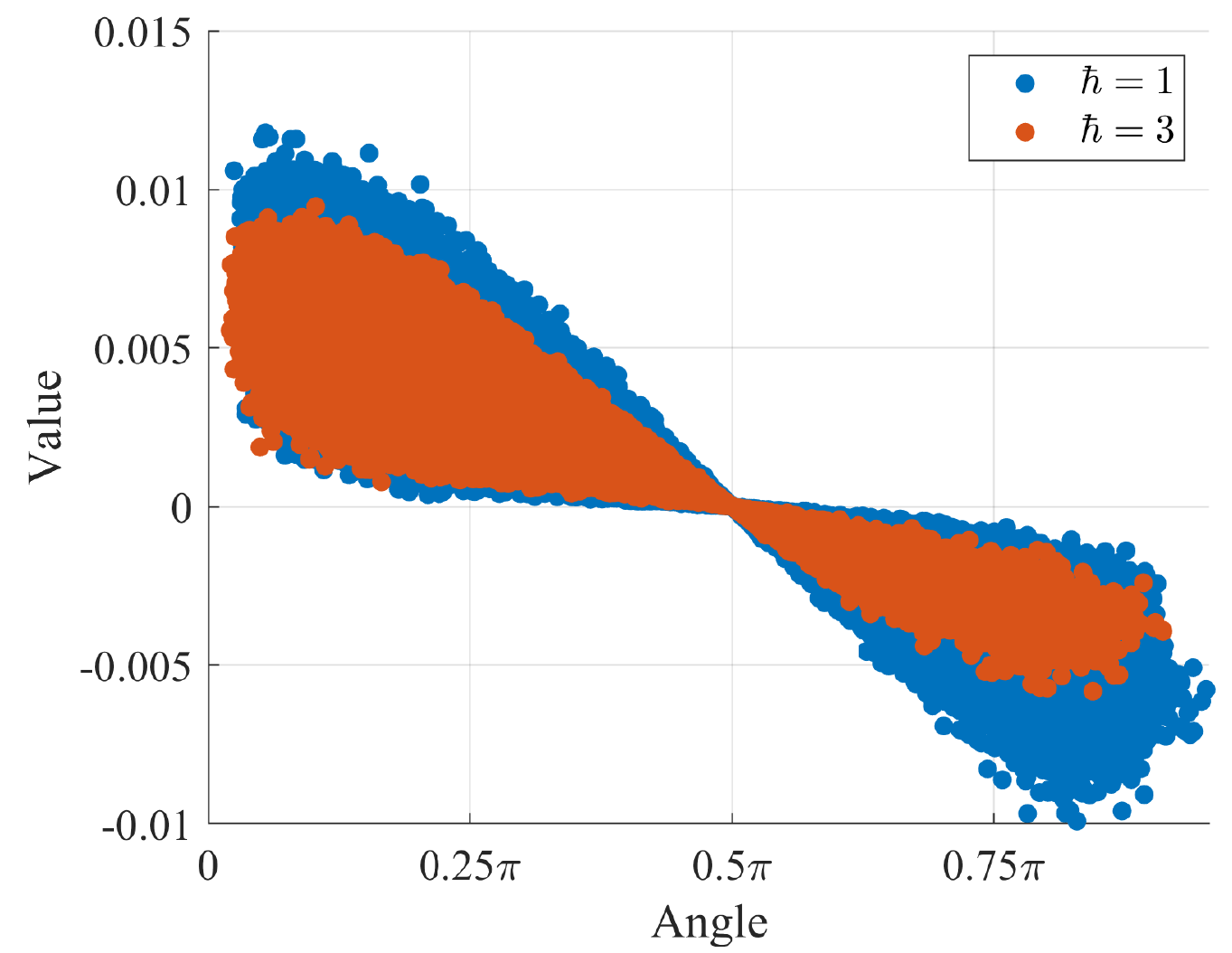}
	\caption{Angularities of NNGP$^{(d)}$ kernels with various $\hbar$.}
	\label{Figure_angular_form}
	\vspace{-0.5cm}
\end{figure}

% 	\begin{table}[hbpt]
	% 		\caption{Test accuracy of regression experiments based on NNGP$^{(d)}$ and NNGP$^{(w)}$ kernels.}
	% 		\begin{center}
		% 			\resizebox{\linewidth}{14mm}{
			% 				\begin{tabular}{c|c|c|c|c}
				% 					\hline
				% 					Model & FMNIST	     & Test accuracy & CIFAR10 & Test accuracy\\
				% 					\hline
				% 					\hline
				% 					NNGP$^{(d)}$  & \multirow{2}{*}{1k} & \textbf{0.345$\pm$0.016} & \multirow{2}{*}{1k} & 0.166$\pm$0.018 \\
				% 					NNGP$^{(w)}$ & & 0.342$\pm$0.021  & & \textbf{0.187$\pm$0.018} \\
				
				% 					\hline
				% 				\end{tabular}
			% 			}
		% 		\end{center}
	% 		\label{tab:nngp_exp}
	% 	\end{table}

% 	hbar = 2: 0.29 0.35 0.268 0.3580

% 	hbar = 3: 0.136 0.2278 0.4070 0.4060

% 	hbar = 4: 0.338, 0.3730 0.2940

\section{RELATED WORK} \label{sec:rw}

\textbf{Deep Learning and Kernel Methods.} There have been great efforts on correspondence between deep neural networks and Gaussian processes. Neal \textit{et al.} \cite{neal1996:priors} presented the seminal work by showing that a one-hidden-layer network of infinite width turns into a Gaussian process. Cho \textit{et al.} \cite{cho:MKMs} linked the multi-layer networks using rectified polynomial activation with compositional Gaussian kernels. Lee \textit{et al.} \cite{lee2017deep} showed that the infinitely wide fully-connected neural networks with commonly-used activation functions can converge to Gaussian processes. Recently, the NNGP has been scaled to many types of networks including Bayesian networks \cite{novak2018bayesian}, deep networks with convolution \cite{garriga2019:bayesian}, and recurrent networks \cite{yang2019:rnn}. Furthermore, Wang \textit{et al.} \cite{wang2020:bridging} wrote an inclusive review for studies on connecting neural networks and kernel learning. Despite great progress, all existing works about NNGP still rely on increasing width to induce the Gaussian processes, yet we go into the depth paradigm and offer an NNGP by increasing depth, which not only complements the existing theory to a good degree but also enhances our understanding to the true picture of ``deep'' learning.

\textbf{Developments of NNGPs.} Recent years have witnessed a growing interest in neural network Gaussian processes. NNGPs can provide a quantitative characterization of how likely certain outcomes are if some aspects of the system are not exactly known. In the experiments of \cite{lee2017deep}, an explicit estimate in the form of variance prediction is given to each test sample. Besides, Pang \textit{et al.} \cite{pang2019neural} showed that the NNGP is good at handling data with noise and is superior to discretizing differential operators in solving some linear or nonlinear partially differential equations. Park \textit{et al.} \cite{park2020towards} employed the NNGP kernel in the performance measurement of network architectures for the purpose of speeding up the neural architecture search. Dutordoir \textit{et al.} \cite{dutordoir2020bayesian} presented the translation insensitive convolutional kernel by relaxing the translation invariance of deep convolutional Gaussian processes. Lu \textit{et al.} \cite{lu2020interpretable} proposed an interpretable NNGP by approximating an NNGP with its low-order moments.

\section{Conclusions and Prospects} \label{sec:conclusions}
In this paper, we have presented the first depth-induced NNGP (NNGP$^{(d)}$) based on a width-depth symmetry consideration. Next, we have characterized the basic properties of the proposed NNGP$^{(d)}$ kernel by proving its uniform tightness and estimating its smallest eigenvalue, respectively. Such results serve as a solid base for the understanding and application of the derived NNGP, such as the generalization and optimization properties and Bayesian inference with the NNGP$^{(d)}$. Lastly, we have conducted regression experiments on image classification and showed that our proposed NNGP$^{(d)}$ kernel can achieve a performance comparable to the NNGP$^{(w)}$ kernel. Future efforts can be put into scaling the proposed NNGP$^{(d)}$ kernel into more applications.

\section{Acknowledgments}
Shaoqun Zhang would like to acknowledge the support from the Program B for Outstanding Ph.D Candidate of Nanjing University (202101B051). Dr. Fei Wang would like to acknowledge the support from Amazon AWS machine learning for research award and Google faculty research award.

\bibliographystyle{unsrt}
\bibliography{reference}
	
% \clearpage

% \begin{IEEEbiography}[{\includegraphics[width=1in,height=1.25in,clip,keepaspectratio]{Figs/zhangsq}}]{Shao-Qun Zhang}
% received his B.Sc. degree in June 2015 and M.Sc. degree in June 2018 from Sichuan University, Chengdu, China. He is currently working toward his Ph.D. degree with the National Key Lab for Novel Software Technology in Nanjing University under the supervision of Prof. Zhi-Hua ZHOU. His current research interests mainly include machine learning and data mining.
% \end{IEEEbiography}
	
% \vspace{11pt}

% \begin{IEEEbiography}[{\includegraphics[width=1in,height=1.25in,clip,keepaspectratio]{Figs/zhangsq}}]{Feng-Lei Fan}

% \end{IEEEbiography}
\section{Appendices}
In this appendix, for self-sufficiency, we will not only show proofs but also restate related theorems and notations. 

\subsection{Uniform Tightness of NNGP$^{(d)}$}

	\begin{lemma}[Lemma 5 in the manuscript] 
		Based on the notations in the manuscript, for any $\boldsymbol{x}, \boldsymbol{x}' \in \mathbb{R}^d$ and $s \in [t]$, there exist $\alpha, \beta, C >0$, such that 
		\[
		\mathbb{E} \left[ \sup_i \big| Z^s_i(\boldsymbol{x}) - Z^s_i(\boldsymbol{x}') \big|^{\alpha} \right] \leq C \| \boldsymbol{x} - \boldsymbol{x}' \|^{\beta+d} .
		\]
	\end{lemma}
	\begin{proof}
		This proof follows mathematical induction. Before that, we show the following preliminary result. Let $\theta$ be one element of the augmented matrix $(\mathbf{W}^l, \boldsymbol{b}^l)$ at the $l$-th layer, then we can formulate its characteristic function as
		\[
		\varphi(t) = \mathbb{E}\left[ e^{\mathrm{i}\theta t} \right] = e^{-\eta^2 t^2/2} \quad\text{with}\quad \theta \sim \mathcal{N}(0,\eta^2) ,
		\]
		where $\mathrm{i}$ denotes the imaginary unit with $\mathrm{i} = \sqrt{-1}$. Thus, the variance of hidden random variables at the $l^{th}$ layer becomes
		\begin{equation} \label{eq:sigma}
			\sigma^2_l = \eta^2 \left[ 1 + \frac{1}{n_l} \sum_{i=1}^{n_l} \big| \varphi \circ Z^{l-1}_i \big|^2 \right] .
		\end{equation}
		Since the activation $\sigma$ is a well-posed function and $(\mathbf{W}^l,\boldsymbol{b}^l) \in SP(\sigma)$, we affirm that $\varphi$ is Lipschitz continuous (with Lipschitz constant $L_{\varphi}$). 
		
		Now we start the mathematical induction. When $s=1$, for any $\boldsymbol{x}, \boldsymbol{x}' \in \mathbb{R}^d$ and $s \in [t]$, we have
		\[
		\mathbb{E} \left[ \sup_i \big| Z^1_i(\boldsymbol{x}) - Z^1_i(\boldsymbol{x}') \big|^{\alpha} \right] \leq C_{\eta,\theta,\alpha} \| \boldsymbol{x} - \boldsymbol{x}' \|^{\alpha} ,
		\]
		where $C_{\eta,\theta,\alpha} = \eta^{\alpha}~ \mathbb{E}[ |\mathcal{N}(0,1)|^{\alpha} ] $. Per mathematical induction, for $s \geq 1$, we have
		\[
		\mathbb{E} \left[ \sup_i \big| Z^s_i(\boldsymbol{x}) - Z^s_i(\boldsymbol{x}') \big|^{\alpha} \right] \leq C_{\eta,\theta,\alpha} \| \boldsymbol{x} - \boldsymbol{x}' \|^{\alpha} .
		\]
		Thus, one has
		\begin{multline} \label{eq:induction}
			\mathbb{E} \left[ \sup_i \big| Z^s_i(\boldsymbol{x}) - Z^s_i(\boldsymbol{x}') \big|^{\alpha} \right] \\
			\leq (C_{\sigma})^{\alpha}~ \mathbb{E}[ |\mathcal{N}(0,1)|^{\alpha} ]~  \big| Z^{s-1}_j(\boldsymbol{x}) - Z^{s-1}_j(\boldsymbol{x}') \big|^{\alpha} ,
		\end{multline}
		where
		\[
		\begin{aligned}
			C_{\sigma} &= \sigma^2_0(\boldsymbol{x}) - 2 \Sigma_{\boldsymbol{x},\boldsymbol{x}'} + \sigma^2_0(\boldsymbol{x}') \\
			=&~ \frac{\eta^2}{n_{s-1}} \sum_{j=1}^{n_{s-1}} \big| \varphi\circ Z_j^{s-1}(\boldsymbol{x}) - \varphi\circ Z_j^{s-1}(\boldsymbol{x}') \big|^2  \quad\text{(from Eq.~\eqref{eq:sigma})} \\
			\leq&~ \frac{\eta^2 L_{\varphi}^2}{n_{s-1}} \sum_{j=1}^{n_{s-1}} \big| Z_j^{s-1}(\boldsymbol{x}) - Z_j^{s-1}(\boldsymbol{x}') \big|^2 .
		\end{aligned}
		\]
		Thus, Eq.~\eqref{eq:induction} becomes
		\[
		\mathbb{E} \left[ \sup_i \big| Z^s_i(\boldsymbol{x}) - Z^s_i(\boldsymbol{x}') \big|^{\alpha} \right] \leq C_{\eta,\theta,\alpha}' \big| Z^{s-1}_j(\boldsymbol{x}) - Z^{s-1}_j(\boldsymbol{x}') \big|^{\alpha} ,
		\]
		where
		\[
		C_{\eta,\theta,\alpha}' = \frac{(\eta L_{\varphi})^{\alpha}}{n_{s-1}} \sum_{j=1}^{n_{s-1}} \big| Z_j^{s-1}(\boldsymbol{x}) - Z_j^{s-1}(\boldsymbol{x}') \big|^{\alpha}~ \mathbb{E}[ |\mathcal{N}(0,1)|^{\alpha} ] .
		\]
		Iterating this argument, we obtain
		\[
		\mathbb{E} \left[ \sup_i \big| Z^s_i(\boldsymbol{x}) - Z^s_i(\boldsymbol{x}') \big|^{\alpha} \right] \leq C_{\eta,\theta,\alpha} \| \boldsymbol{x} - \boldsymbol{x}' \|^{\alpha} ,
		\]
		where 
		\[
		C_{\eta,\theta,\alpha} = \eta^{\alpha (s+1)} L_{\varphi}^{\alpha s} ~ \mathbb{E}[ |\mathcal{N}(0,1)|^{\alpha} ]^{s+1} .
		\]
		The above induction holds for any positive even $\alpha$. Let $\beta = \alpha - d > 0$, then this lemma is proved as desired.
	\end{proof}
	
\subsection{Tight Bound for the Smallest Eigenvalue}

	\begin{theorem}[Theorem 3 in the manuscript] 
		Suppose that $\boldsymbol{x}_1, \dots, \boldsymbol{x}_N$ are i.i.d. sampled from $P_X = \mathcal{N}(0,\eta^2)$ and $P_X$ is a well-scaled distribution, then for an integer $r \geq 2$, with probability $1-\delta>0$, we have $\lambda_{\min} \left( \mathbf{K}_{\mathcal{D},\mathcal{D}} \right) = \Theta(d)$, where
		\[
		\delta \leq N e^{-\Omega(d)} + N^{2} e^{-\Omega\left(d N^{-2 /(r-0.5)}\right)} .
		\]
	\end{theorem}

We begin this proof with the following lemmas.
	\begin{lemma} \label{lemma:concentration}
		Let $f:\mathbb{R}^d \to \mathbb{R}$ be a Lipschitz continuous function with constant $L$ and $P_X$ denote the Gaussian distribution $\mathcal{N}(0,\eta^2)$, then for $\forall~ \delta>0$, there exists $c >0$, s.t.
		\begin{equation} \label{eq:log}
			\mathbb{P} \left( \left| f(\boldsymbol{x})-\int f \left(\boldsymbol{x}^{\prime}\right) \dif P_{X}\left(\boldsymbol{x}^{\prime}\right) \right| > \delta \right) \leq 2 e^\frac{-c \delta^{2}}{L^2} .
		\end{equation}
	\end{lemma}
	Lemma~\ref{lemma:concentration} shows that the Gaussian distribution corresponding to our samples satisfies the log-Sobolev inequality (i.e., Eq.~\eqref{eq:log}) with some constants unrelated to dimension $d$. This result also holds for the uniform distributions on the sphere or unit hypercube~\cite{nguyen2021:eigenvalues}. 
	
	\begin{lemma} \label{lemma:input}
		Suppose that $\boldsymbol{x}_1, \dots, \boldsymbol{x}_N$ are i.i.d. sampled from $\mathcal{N}(0,\eta^2)$, then with probability $1-\delta>0$, we have 
		\[
		\|\boldsymbol{x}_i\|_2 = \Theta(\sqrt{d}) \quad \text{and}\quad |\langle \boldsymbol{x}_i,\boldsymbol{x}_j \rangle|^r \leq dN^{-1/(r-0.5)}, 
		\]
		for $i \neq j$, where
		\[
		\delta \leq N e^{-\Omega(d)} + N^{2} e^{-\Omega\left(d N^{-2 /(r-0.5)}\right)} .
		\]
	\end{lemma}
	\begin{proof}
		From Definition 1 of the manuscript, we have
		\[
		\int\|\boldsymbol{x}\|_{2}^{2} \dif P_X(\boldsymbol{x}) = \Theta(d).
		\]
		Since $\boldsymbol{x}_1, \dots, \boldsymbol{x}_n$ are i.i.d. sampled from $P_X = \mathcal{N}(0,\eta^2)$, for $\forall$ $i\in[N]$, we have $\|\boldsymbol{x}_i\|_2^2 = \Theta(d)$ with probability at least $1- Ne^{\Omega(d)}$. Provided $\boldsymbol{x}_i$, the single-sided inner product $\langle \boldsymbol{x}_i,\cdot \rangle$ is Lipschitz continuous with the constant $L = \mathcal{O}(\sqrt{d})$. As such, from Lemma~\ref{lemma:concentration}, for $\forall~ j\neq i$, we have
		\[
		\mathbb{P} \left( |\langle \boldsymbol{x}_i,\boldsymbol{x}_j \rangle| > \delta^* \right) \leq 2 e^{-\delta^2/L^2} .
		\]
		Then, for $r \geq 2$, we have
		\[
		\mathbb{P} \left( \max_{j\neq i} |\langle \boldsymbol{x}_i,\boldsymbol{x}_j \rangle|^r > \delta^* \right) \leq N^{2} e^{-\Omega\left( {\delta^*}^2\right)} .
		\]
		We complete the proof by setting $\delta^* \leq dN^{-1/(r-0.5)}$.
	\end{proof}

	\noindent\emph{\textbf{Proof of Theorem~\ref{thm:smallest}.}} We start this proof with some notations. Recall the empirical NNGP$^{(d)}$ kernel $\mathbf{K}_{\mathcal{D},\mathcal{D}}$. For convenience, we force $n^* =|\boldsymbol{z}_1|_{\#} = |\boldsymbol{z}_2|_{\#} = \dots = |\boldsymbol{z}_L|_{\#}$. We also abbreviate the covariance $\mathrm{Cov}(\boldsymbol{z}^{l_1+\kappa\hbar},\boldsymbol{z}^{l_1+\kappa\hbar})$ as $\mathbf{C}_{l_1+\kappa\hbar}$ and pick $l_1 = 1$ throughout this proof.
	
	Unfolding the NNGP$^{(d)}$ kernel equation 
		\begin{equation}
		\mathbf{K}(\boldsymbol{x}_i,\boldsymbol{x}_j) = \mathbb{E}[\langle f(\boldsymbol{x}_i; \boldsymbol{\theta}), f(\boldsymbol{x}_j; \boldsymbol{\theta})\rangle], \quad\text{for}\quad \boldsymbol{x}_i, \boldsymbol{x}_j \in \mathcal{D},
	\end{equation}
	we have 
	\begin{equation} \label{eq:K_cov}
		\mathbf{K}(\boldsymbol{x}_i, \boldsymbol{x}_j) = \frac{1}{M_{\boldsymbol{z}}} \left[ \sum_{\kappa} \varphi_{\kappa} + \sum_{\kappa_1 \neq \kappa_2} \phi_{\kappa_1,\kappa_2} \right] ,
	\end{equation}
	where
	\[
	\left\{\begin{aligned}
		&\varphi_{\kappa} = \mathbb{E} \left[ \langle \boldsymbol{z}^{l_1+\kappa\hbar}, \boldsymbol{z}^{l_1+\kappa\hbar} \rangle \right] , \\
		&\phi_{\kappa_1,\kappa_2} = \sum\nolimits_{p,q} \mathbb{E} \left[ \boldsymbol{z}_p^{l_1+\kappa_1\hbar} \boldsymbol{z}_q^{l_1+\kappa_2\hbar} \right], \quad\text{for}\quad \kappa_1 \neq \kappa_2 ,
	\end{aligned} \right.
	\]
	in which the subscript $p$ indicates the $p$-th element of vector $\boldsymbol{z}^{l_1+\kappa_1\hbar}$. From Theorem 1 of the manuscript, the sequence of random variables $\{ \boldsymbol{z}^{l_1}, \boldsymbol{z}^{l_1+\hbar}, \dots, \boldsymbol{z}^{l_1+\kappa\hbar} \}$ is weakly dependent with $\beta(s) \to \infty$ as $s\to \infty$. Thus, $\phi_{\kappa_1,\kappa_2}$ is an infinitesimal with respect to $l_1+|\kappa_2-\kappa_1|\hbar$ when $\kappa_1 \neq \kappa_2$ and $\hbar$ is sufficiently large.
	
	\noindent Invoking the following equations
		\begin{equation} 
		\left\{\begin{aligned}
			&\lambda_{\min}(\mathbf{P}\mathbf{Q}) \geq \lambda_{\min}(\mathbf{P}) \min_{i\in[m]} \mathbf{Q}(i,i) \ , \\
			&\lambda_{\min}(\mathbf{P}+\mathbf{Q}) \geq \lambda_{\min}(\mathbf{P}) + \lambda_{\min}(\mathbf{Q}) 
		\end{aligned} \right.
	\end{equation}
	into Eq.~\eqref{eq:K_cov}, we have
	\begin{align}
		\begin{split} \label{eq:thm2_1}
			&\lambda_{\min}(\mathbf{K}_{\mathcal{D},\mathcal{D}}) \geq \sum\nolimits_{\kappa} \lambda_{\min} \left( \mathbf{C}_{l_1+\kappa\hbar} \right),
		\end{split}\\
		\begin{split} \label{eq:thm2_2}
			&\lambda_{\min} \left( \mathbf{C}_{l_1+\kappa\hbar}\right) \geq \lambda_{\min} \left( \mathbf{C}_{l_1+\kappa\hbar-1} \right), \quad\text{for}\quad \kappa \in \mathbb{N}.
		\end{split}
	\end{align}
	Iterating Eq.~\eqref{eq:thm2_2} and then invoking it into Eq.~\eqref{eq:thm2_1}, we have
	\begin{equation} \label{eq:thm2_3}
		\lambda_{\min}(\mathbf{K}_{\mathcal{D},\mathcal{D}}) \geq \sum\nolimits_{\kappa} \lambda_{\min}\left( \mathbf{C}_1 \right) .
	\end{equation}
	From the Hermite expansion~\cite{zhang2021:arise} of ReLU function, we have
	\begin{equation} \label{eq:relu}
		\mu_{r}(\sigma) = (-1)^{\frac{r-2}{2}} (r-3)!! / \sqrt{2 \pi r!} \ ,
	\end{equation}
	where $r \geq 2$ indicates the expansion order. Thus, we have
	\begin{equation} \label{eq:thm2_4}
		\begin{aligned}
			& \lambda_{\min}\left( \mathbf{C}_1 \right) = \lambda_{\min}\left( \sigma(\mathbf{W}^1\mathbf{X}) \sigma(\mathbf{W}^1\mathbf{X})^{\top} \right) \\
			&\geq \mu_{r}(\sigma)^2 \lambda_{\min}\left( \mathbf{X}^{(r)} \left( \mathbf{X}^{(r)} \right)^{\top} \right) \\
			&\geq \mu_{r}(\sigma)^2 \left( \min_{i\in[N]} \|\boldsymbol{x}_i\|_2^{2r} - (N-1) \max_{j\neq i} |\langle \boldsymbol{x}_i,\boldsymbol{x}_j \rangle|^r \right) \\
			&\geq \mu_{r}(\sigma)^2 \Omega(d) \ ,
		\end{aligned}
	\end{equation}
	where the superscript $(r)$ denotes the $r$-th Khatri Rao power of the matrix $\mathbf{X}$, the first inequality follows from Eq.~\eqref{eq:relu}, the second one holds from Gershgorin Circle Theorem~\cite{salas1999gershgorin}, and the third one follows from Lemma~\ref{lemma:input}. Therefore, we can obtain the lower bound of the smallest eigenvalue by plugging Eq.~\eqref{eq:thm2_4} into Eq.~\eqref{eq:thm2_3}
	
	On the other hand, it's observed from Lemma 1 of the manuscript that for $l \in [L]$,
	\begin{equation} \label{eq:thm2_5}
		\left\{\begin{aligned}
			& \| \boldsymbol{z}_p^l \|^2_2 = \mathbb{E}_{\mathbf{W}^l_p} \left[ \sigma(\mathbf{W}^l_p\boldsymbol{z}^{l-1})^2 \right] = \| \boldsymbol{z}_q^l \|^2, \quad\text{for}\quad \forall q \neq p ,\\
			&\| \boldsymbol{z}^l \|_2^2 = \mathbb{E}_{\mathbf{W}^l} \left[ \sigma(\mathbf{W}^l\boldsymbol{z}^{l-1})^2 \right] \leq \| \boldsymbol{z}^l \|_2^2. 
		\end{aligned}\right.
	\end{equation}
	Thus, we have
	\[
	\begin{aligned}
		\lambda_{\min}(\mathbf{K}_{\mathcal{D},\mathcal{D}}) &\leq \frac{\tr(\mathbf{K}_{\mathcal{D},\mathcal{D}})}{N} = \frac{1}{N} \sum_{i}^{N} \mathbf{K}(\boldsymbol{x}_i, \boldsymbol{x}_i) \\
		& \leq  \frac{1}{N} \sum_{i}^{N} \frac{1}{M_{\boldsymbol{z}}} \left[ \sum_{\kappa} \varphi_{\kappa} + \sum_{\kappa_1 \neq \kappa_2} \phi_{\kappa_1,\kappa_2} \right] \\
		&\leq \frac{1}{N} \sum_{i}^{N}\left( \frac{1}{\kappa} \sum_{\kappa} \max_{j\in[N]} \|\boldsymbol{x}_j\|_2^2 + \Omega(d) \right) \\
		&\leq \Theta(d) ,
	\end{aligned}
	\]
	where the second inequality follows from Eq.~\eqref{eq:K_cov}, the third one follows from Eq.~\eqref{eq:thm2_5}, and the fourth one holds from Lemma~\ref{lemma:input}. This completes the proof. $\hfill\square$
	
\subsection{Analysis Experiments}

\ff{To have a better understanding of the proposed NNGP$^{(d)}$ kernel, here we explore the impacts of the separation $\hbar$, the number of samples, the parameter variance, and the network size on it, as well as the computation time. Now we introduce them one by one.}

\begin{table}[th]
	\caption{\ff{Test accuracy on the FMNIST test data by the NNGP$^{(d)}$ kernel induced with different $\hbar$.}}
	\begin{center}
			\ff{\begin{tabular}{|c|c|}
					\hline
					 $\hbar$	     & Test accuracy \\
					\hline
					 1 & 0.329$\pm$0.027 \\
					\hline
					 2 & 0.198$\pm$0.027\\
					\hline
					 3 &  0.179$\pm$0.022\\	
					\hline					 4 & 0.145$\pm$0.014 \\		
					\hline
					 5 &  0.146$\pm$0.017 \\
					\hline
				\end{tabular}}
		\end{center}
		\label{tab:hbar}
	\end{table}	
\ff{\textbf{The impact of $\hbar$.} We set the network depth to $200\times\hbar$ so that the NNGP$^{(d)}$ kernel is empirically computed by aggregating $\kappa=200$ shortcut connections with a separation of $\hbar$. For a comprehensive comparison, $\hbar$ is selected from $\{1,2,3,4,5\}$. Next, the NNGP$^{(d)}$ kernels are constructed with these networks and FMNIST-4k training data. All parameters are initialized with a mean of $0$ and a variance of $0.5$. Table \ref{tab:hbar} demonstrates the testing performance of so-built NNGP$^{(d)}$ kernels with respect to the FMNIST test data. As suggested by our angular plot analysis in the main body, the kernel with a larger $\hbar$ is compressed to be closer to zero relative to the kernel with a lower $\hbar$. Correspondingly, the kernel with a larger $\hbar$ should have lower discriminative ability. Table \ref{tab:hbar} shows that a larger $\hbar$ leads to an inferior test accuracy, which agrees with our analysis. We conclude that the separation $\hbar$ should be set to a smaller number to make a powerful NNGP$^{(d)}$ kernel.}

\ff{\textbf{The impact of number of samples.} Here we investigate the impact of the number of training samples on the model's performance. We still conduct regression experiments on FMNIST and CIFAR-10 data sets. Following the configurations in the main body, we respectively sample 1k, 2k, 3k, 4k, 5k data from the training sets to construct NNGP$^{(d)}$ and NNGP$^{(w)}$ kernels. Figure \ref{Figure_number_sample} shows the testing accuracy and its associated error bars of two kernels on FMNIST and CIFAR-10. Regarding the NNGP$^{(d)}$ kernel, its test accuracy culminates at 3k for both datasets. While for NNGP$^{(w)}$ kernel, the maximum accuracy is reached at 2k and 3k for FMNIST and CIFAR-10, respectively. We conclude that NNGP$^{(w)}$ and NNGP$^{(d)}$ kernels have a similar performance-sample behavior. }	

\begin{figure}[ht]
	\centering
	\includegraphics[width=\linewidth]{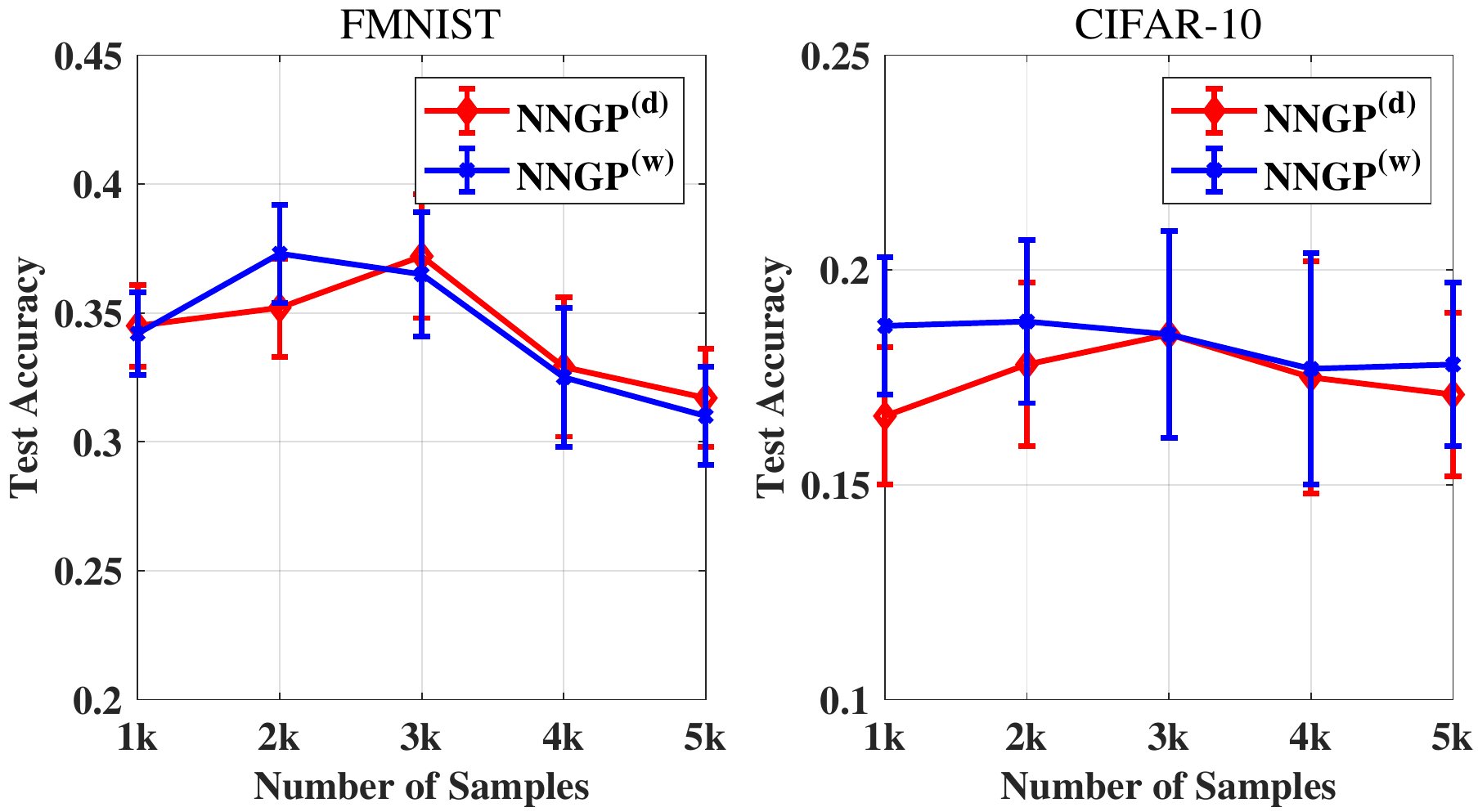}
	\caption{The performance of the NNGP$^{(d)}$ and NNGP$^{(w)}$ kernels constructed by different number of samples on FMNIST and CIFAR-10.}
	\label{Figure_number_sample}
\end{figure}

\begin{table*}[htb]
 \centering
\caption{\ff{The MSE scores of NNGP$^{(d)}$ and NNGP$^{(w)}$ kernels with respect to different variances and network sizes on the synthetic dataset.}}
\scalebox{0.7}{
 \ff{\begin{tabular}{ |c|c|c|c|c|c|c|c|c|c|  }
 \hline
\multirow{3}{4em}{NNGP$^{(w)}$}  & \multicolumn{3}{c|}{width=500}  & \multicolumn{3}{c|}{width=1000} &  \multicolumn{3}{c|}{width=2000} \\
\cline{2-10}
 & $\sigma=0.3$ & $\sigma=0.5$ & $\sigma=0.8$ & $\sigma=0.3$ & $\sigma=0.5$ & $\sigma=0.8$ & $\sigma=0.3$ & $\sigma=0.5$ & $\sigma=0.8$   \\
\cline{2-10}
   & 0.1240$\pm$0.0279 &0.1051$\pm$0.0317 & 0.1143$\pm$0.0208 & 0.1350$\pm$0.0243 & 0.1092$\pm$0.0438 &0.1075$\pm$0.0389& 0.1255$\pm$ 0.0158& 0.0920$\pm$ 0.0444& 0.1006$\pm$0.0275 \\
  \hline
\multirow{3}{4em}{NNGP$^{(d)}$}  & \multicolumn{3}{c|}{depth=100}  & \multicolumn{3}{c|}{depth=200} &  \multicolumn{3}{c|}{depth=500} \\
\cline{2-10}
 & $\sigma=0.2$ & $\sigma=0.3$ & $\sigma=0.5$ & $\sigma=0.2$ & $\sigma=0.3$ & $\sigma=0.5$ & $\sigma=0.2$ & $\sigma=0.3$ & $\sigma=0.5$   \\
\cline{2-10}
  & 0.1437$\pm$ 0.0217& 0.1808
    $\pm$0.0443 &  0.2184$\pm$0.0188 & 0.1310$\pm$0.0635 &  0.1926$\pm$0.0436 & 0.1742$\pm$0.0442 &     0.0917$\pm$0.0391 & 0.2056$\pm$0.0304 &0.2384$\pm$0.1308\\
  \hline
   \end{tabular}}}
 \vspace{1ex}
  \label{tab:sinc}
\end{table*}

\ff{\textbf{The impact of parameter variance and network size.} We construct a synthetic data set to investigate the impact of variance and network size on the model's performance. The task is to use the NNGP$^{(w)}$ and NNGP$^{(d)}$ kernels to fit a function: $f(x)=\sin(x)$ over $[0,\pi]$. A total of 200 data points are evenly sampled from $[0,\pi]$, from which 100 points are randomly sampled for training and the rest for testing.}

\ff{Similarly, we employ a one-hidden-layer wide network for computing the NNGP$^{(w)}$ kernel whose width is cast from $\{500,1000,2000\}$. In contrast, we use a deep network for the NNGP$^{(d)}$ kernel whose depth is cast from $\{100, 200, 500\}$. The width of the deep network is set to 30, and $\hbar=1$. No label encoding is needed here because this is not a classification task. For two networks, we take $\tanh$ as the activation function. For the NNGP$^{(w)}$, all weights are initialized with a Gaussian distribution of mean $0$ and variance of $\{0.3, 0.5, 0.8\}/n_{l}$, where $n_{l}$ is the number of neurons in the $l$-th layer. For NNGP$^{(d)}$, all weights are initialized with a Gaussian distribution of mean $0$ and variance of $\{0.2, 0.3, 0.5\}$. The initialization is repeated $200$ times to compute the empirical statistics of the NNGP$^{(d)}$ and NNGP$^{(w)}$. We run each model $10$ times to count the mean and variance of accuracy. All experiments are conducted on an NVIDIA TITAN Xp GPU. }

\begin{figure}[h]
	\centering
	\includegraphics[width=\linewidth]{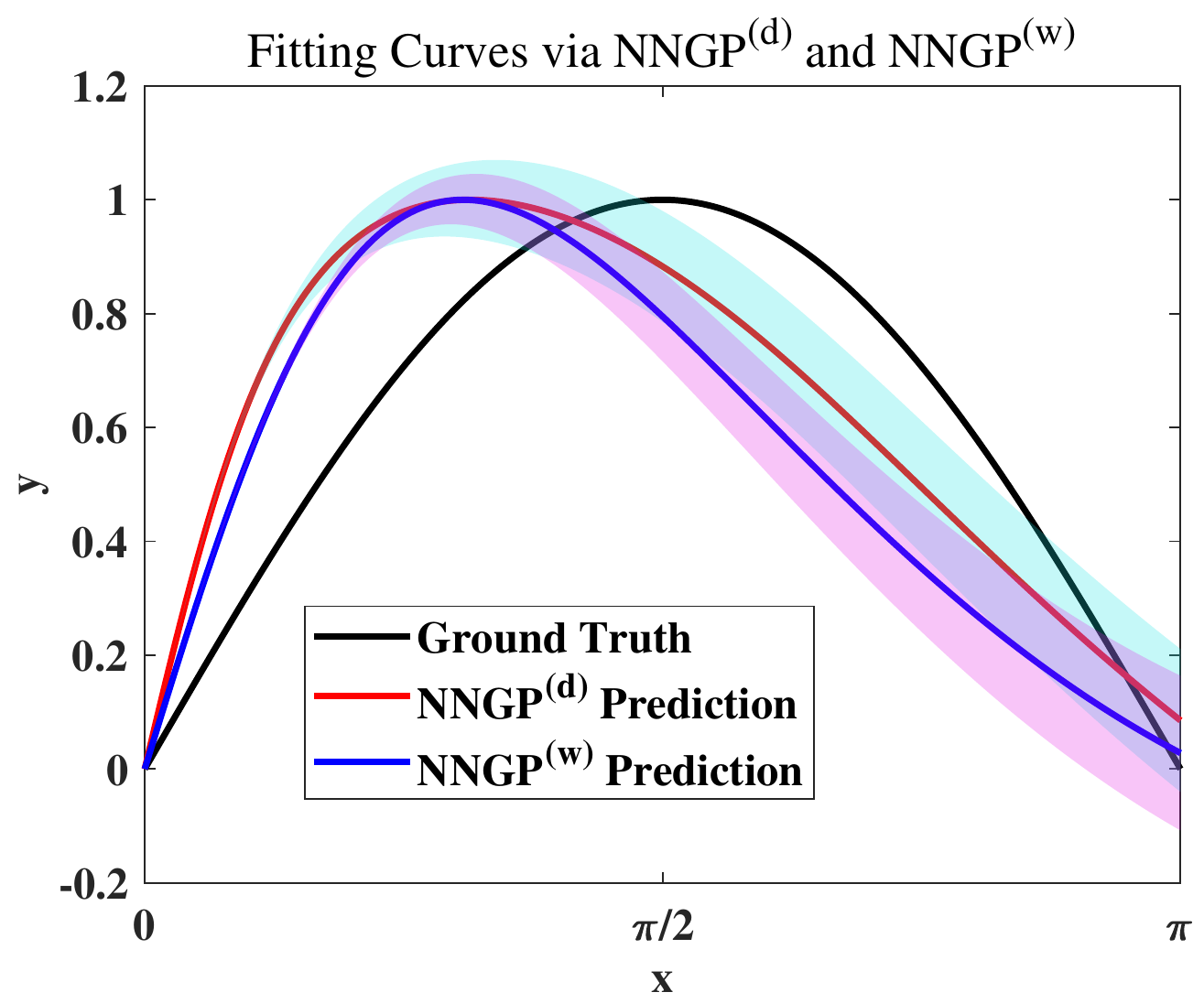}
	\caption{The fitting curves of the NNGP$^{(d)}$ and NNGP$^{(w)}$ kernels for a sine function over $[0,\pi]$.}
	\label{Figure_fitting}
\end{figure}

\ff{Table \ref{tab:sinc} shows the performance of NNGP$^{(d)}$ and NNGP$^{(w)}$ kernels with respect to different variances and network sizes, from which we draw two highlights. The first is that with the same network size, the NNGP$^{(d)}$ kernel favors a lower variance, while the NNGP$^{(w)}$ kernel is on the contrary. The second one is that increasing the network size may not necessarily give rise to a lower MSE. In fact, it depends on the variance. For the NNGP$^{(d)}$ kernel, when $\sigma=0.2$, increasing the network depth is beneficial, whereas when $\sigma=0.3$, increasing the network depth hurts the performance. Figure \ref{Figure_fitting} presents the fitting curves of the NNGP$^{(d)}$ ($\sigma=0.2$, depth=200) and NNGP$^{(w)}$ ($\sigma=0.5$, width=1000) kernels for $\sin(x)$, where the curve of NNGP$^{(w)}$ is more accurate in $[0,\pi/2]$ and the curve of NNGP$^{(d)}$ is more accurate in $[\pi/2, \pi]$.}

\begin{table}[h]
	\caption{\ff{The computation time in constructing the NNGP$^{(d)}$ and NNGP$^{(w)}$ kernels.}}
	\begin{center}
			\ff{\begin{tabular}{|c|c|c|}
					\hline
					 \#Sample	   & NNGP$^{(d)}$ &  NNGP$^{(w)}$\\
					\hline
					 1k & 0.011s & 0.057s  \\
					 2k & 0.173s & 0.083s \\
					 3k & 0.228s & 0.160s \\
					 4k & 0.286s & 0.226s \\
					 5k & 0.444s & 0.342s \\
					\hline
				\end{tabular}}
		\end{center}
		\label{tab:time}
	\end{table}	

\ff{\textbf{Computation time.} Here, we also compare the time spent on constructing the NNGP$^{(d)}$ and NNGP$^{(w)}$ kernels relative to different numbers of samples. We sample the data from FMNIST. The network size is $200$ for both deep and wide networks. $\hbar$ is $1$ for the deep network. Previously, we repeat the initialization 200 times to compute a kernel. Here, the repetition time is set to 1 for convenience. The experiment is conducted on Intel Core-i7-6500U. As shown in Table \ref{tab:time}, generally, it is more expensive to construct the NNGP$^{(d)}$ kernel than the NNGP$^{(w)}$ kernel. However, the difference in computation time is no more than $2\times$, as the number of samples increases. This might be because we use the CPU which does not admit parallel acceleration.}

\end{document}